\documentclass[11pt,a4paper]{article}
\usepackage[hyperref]{eacl2021}
\usepackage{times}
\usepackage{danudefs}
\usepackage{latexsym}
\usepackage{microtype}
\usepackage{amsmath}
\usepackage{booktabs}
\usepackage{graphics}
\usepackage{mathtools}
\usepackage{braket}
\usepackage{adjustbox}
\usepackage{subcaption}
\usepackage{caption}
\usepackage{hyperref}

\newtheorem{lemma}[]{Lemma}

\newcommand{\expip}[3]{\exp \left(\vec{#1}\T\mat{R}_{#2}\vec{#3} \right)}
\newcommand{\Max}[1]{\raisebox{0.5ex}{\scalebox{0.8}{$\displaystyle \max_{#1}\,\,$}}}

\usepackage{microtype}

\aclfinalcopy % Uncomment this line for the final submission
\title{RelWalk -- A Latent Variable Model Approach to \\Knowledge Graph Embedding}

\author{ Danushka Bollegala$^{1}$\thanks{Danushka Bollegala holds concurrent appointments as a Professor at University of Liverpool and as an Amazon Scholar. This paper describes work performed at the University of Liverpool and is not associated with Amazon.}, Huda Hakami$^{2}$, Yuichi Yoshida$^{3}$ \and Ken-ichi Kawarabayashi$^{3}$ \\
$^1$University of Liverpool, Amazon. \texttt{danushka@liverpool.ac.uk}\\
$^2$Taif University \texttt{hahakami@tu.edu.sa}\\
$^3$National Institute of Informatics  \texttt{\{yyoshida,k\_keniti\}@nii.ac.jp}}

\date{}

\begin{document}
\maketitle
\begin{abstract}
Embedding entities and relations of a knowledge graph in a low-dimensional space has shown impressive performance in predicting missing links between entities. 
Although progresses have been achieved, existing methods are heuristically motivated and theoretical understanding of such embeddings is comparatively underdeveloped. 
This paper extends the random walk model~\cite{Arora:TACL:2016} of word embeddings to Knowledge Graph Embeddings (KGEs) to derive a scoring function that evaluates the strength of a relation $R$ between  two entities $h$ (head) and $t$ (tail).
Moreover, we show that marginal loss minimisation, a popular objective used in much prior work in KGE,
follows naturally from the log-likelihood ratio maximisation under the probabilities estimated from the KGEs according to our theoretical relationship. 
We propose a learning objective motivated by the theoretical analysis to learn KGEs from a given knowledge graph.
Using the derived objective, accurate KGEs are learnt  from FB15K237 and WN18RR benchmark datasets, providing empirical evidence in support of the theory.
\end{abstract}

\section{Introduction}
Knowedge graphs (KGs) such as Freebase~\cite{Freebase} organise information in the form of graphs, where entities are represented by the vertices and the relations between two entities are represented by the edges that connect the corresponding vertices.
Despite the best efforts to create complete and large-scale KGs, most KGs remain incomplete and do not represent all the relations that exist between entities~\cite{min2013distant}. 
In particular, new entities are constantly being generated, and new relations are formed between new as well as existing entities. 
Therefore, it is unrealistic to assume that a real-world KG would be complete at any given time point.
Developing approaches for KG completion is an important research field associated with KGs.

KG components can be embedded into numerical formats by learning representations (a.k.a embeddings) for the entities and relations in a given KG.
The learnt KGEs can be used for \emph{link prediction}, which is the task of predicting whether a particular relation exists between two given entities in the KG. 
Specifically, given KGEs for entities and relations, in link prediction, we predict $R$ that is most likely to exist between $h$ and $t$ according to some scoring formula.
Thus, by embedding entities and relations that exist in a KG in some (possibly lower-dimensional and latent) space, we can infer previously unseen relations between entities, thereby expanding a given KG. 
%~\cite{Nickel:AAAI:2016,Yang:ICLR:2015,Lin:AAAI:2015,Nickerl:ICML:2011,ComplexEmb,KBsurvey,Bordes:AAAI:2011}. 

KGE can be seen as a two-step process.
Given a KG represented by a set of relational triples $(h,R,t)$, where a semantic relation $R$ holds between a head entity $h$ and a tail entity $t$, first a scoring function is defined that measures the \emph{relational strength} of a triple $(h,R,t)$.
Second,the entity and relation embeddings that optimise the defined scoring function are learnt using some optimisation method.
Despite the wide applications of entity and relation embeddings created via KGE methods, the existing scoring functions are heuristically motivated to capture some geometric requirements of the embedding space.
For example, TransE~\cite{Bordes:AAAI:2011} assumes that the entity and relation embeddings co-exist in the same (possibly lower dimensional) vector space and translating (shifting)  the head entity embedding by the relation embedding must make it closer to the tail entity embedding, whereas ComplEx~\cite{ComplexEmb} models the asymmetry in relations using the component-wise multi-linear inner-product among entity and relation embeddings. 
% Relational triples extracted from a given KG are used as positive training instances, whereas pseudo-negative~\cite{Bordes:AAAI:2011} instances are automatically generated by randomly corrupting positive instances.

Theoretical understanding of KGE methods is under developed.
For example, it is not clear how the heuristically defined KGE objectives relate to the generative process of a KG.
Providing such a theoretical understanding of the KGE process will enable us to develop KGE methods that address the weaknesses in the existing KGE methods.
For this purpose, we propose Relational Walk (\textbf{RelWalk}),  a theoretically motivated generative approach for learning KGEs.
We are particularly interested in the semantic relationships that exist between entities such as the \textsf{is-CEO-of} relation between a 
person such as \textbf{Jeff Bezos} and a company such as \textbf{Amazon Inc.}

We model KGE as a random walk over the KG.
Specifically, a random walker at the vertex corresponding to the (head) entity $h$ will uniformly at random select one of the outgoing edges corresponding to the semantic relation $R$, which will lead it to the vertex corresponding to the (tail) entity $t$.
Continuing this random walk will result in a traversal over a path in the KG.
Based on this random walk model we derive a relationship between the probability of $R$ holding between $h$ and $t$, $p(h,t \mid R)$, and their KGEs $\mat{R}$, $\vec{h}$ and $\vec{t}$.
Interestingly, the derived relationship is not covered by any of the previously proposed heuristically-motivated scoring functions,
providing the first-ever KGE method with a provable generative explanation.

We show that the \emph{margin loss}, a popular training objective in prior work on KGE, naturally emerges as the log-likelihood ratio computed from the derived $p(h,t \mid R)$.
Based on this result, we derive a training objective that is optimised for learning KGEs that satisfy our theoretical relationship.
This enables us to empirically verify the theoretical relationships that we derived from the proposed random walk process.

Using  FB15K237 and WN18RR benchmarks, we evaluate the learnt KGEs on link prediction and triple classification.
Although we do not obtain state-of-the-art (SoTA) performance on these benchmark datasets, KGEs learnt using RelWalk perform consistently well on both tasks, providing empirical support to the theoretical analysis conducted in this paper.
We re-emphasise that our main objective in this paper is to study KGEs from an interpretable theoretical perspective and not necessarily improving SoTA.
To this end, we study the relationship between the concentration of the partition function as predicted by our theoretical analysis and the performance of the learnt KGEs.
We observe that when the partition function is narrowly distributed, we are able to learn accurate KGEs.
Moreover, we empirically verify that the learnt relation embedding matrices satisfy the orthogonality property as expected by the theoretical analysis.

\section{Related Work}
\label{sec:related}
At a high-level of abstraction, KGE methods can be seen as differing in their design choices for the following two main problems:
(a) how to represent entities and relations, and
(b) how to model the interaction between two entities and a relation that holds between them.
Next, we briefly discuss prior proposals to those two problems (refer to~\citet{KBsurvey,Nguyen:2017aa,Nickel:2016} for an extended survey on KGE).

A popular choice for representing entities is to use vectors~\cite{Bordes:NIPS:2013,ji-EtAl:2015:ACL-IJCNLP,Yang:ICLR:2015}, whereas relations have been represented by vectors, matrices~\cite{Bordes:AAAI:2011,Nguyen:NAACL:2016,Nickerl:ICML:2011} or tensors~\cite{Socher:NIPS:2013b}.
ComplEx~\cite{ComplexEmb} introduced complex vectors for KGEs to capture the asymmetry in semantic relations.
\citet{Ding:2018aa} further improved CompIEx by imposing non-negativity and entailment constraints to ComplEx.

Given entity and relation embeddings, a scoring function evaluates the strength of a triple $(h,R,t)$.
Scoring functions that encode various intuitions have been proposed such as the $\ell_{1}$ or $\ell_{2}$ norms of the vector formed by a translation of the head entity embedding by the relation embedding over the target embedding, or by first performing a projection from the entity embedding space to the relation embedding space~\cite{Yoon:2016aa}. 
%For example, to capture transitivity, \cite{Yoon:2016aa} proposed \emph{Logical Property Preserving} (LPP) embeddings, where a separate relation-specific mapping matrices are learnt for projecting head and tail entities to the relation space.
As an alternative to using vector norms as scoring functions, DistMult~\cite{Yang:ICLR:2015} and ComplEx~\cite{ComplexEmb} use the component-wise multi-linear dot product. 
\citet{Lacroix:2018aa} proposed the use of nuclear 3-norm regularisers instead of the popular Frobenius norm for canonical tensor decomposition. 
\autoref{tbl:scores} shows the scoring functions along with algebraic structures for entities and relations proposed in selected prior work in KGE learning.
Given a scoring function, KGEs are learnt that assign higher scores to relational triples in existing KGs over triples where the relation does not hold (negative triples) by minimising a loss function such as the logistic loss (RESCAL, DistMult, ComplEx) or marginal loss (TransE). 
% Pseudo negative triples can be generated either by perturbing positive tuples or by adversarial learning~\cite{KBGAN}.
%However, uniformly sampled negative triples are likely to be trivial examples that do not provide much information to the learning process and can be detected by simply checking for the type of the entities in a triple.
%\newcite{KBGAN} proposed an adversarial learning approach, where a \emph{generator} assigns a probability to each relation triple and negative instances are sampled according to this probability distribution to train a \emph{discriminator} that discriminates between positive and negative instances.
%\citet{xiao-huang-zhu:2016:P16-1} proposed TransG, a generative model based on the Chinese restaurant process, to model multiple relations that exist between a pair of entities. However, their relation embeddings are designed to satisfy vector translation similar to TransE.

Alternatively to directly learning embeddings from a graph, several methods~\cite{node2vec,DeepWalk,rdf2vec} have considered the vertices visited during truncated random walks over the graph as \emph{pseudo sentences}, and have applied popular word embedding learning algorithms such as continuous bag-of-words model~\cite{Milkov:2013} to learn vertex embeddings. However, pseudo sentences generated in this manner are syntactically very different from sentences in natural languages.

On the other hand, our work extends the random walk analysis by~\citet{Arora:TACL:2016} that derives a useful connection between the joint co-occurrence probability of two words and the $\ell_{2}$ norm of the sum of the corresponding word embeddings.
Specifically, they proposed a latent variable model where the words in a corpus are generated by a probabilistic model parametrised by a time-dependent discourse vector that performs a random walk.
In contrast to Arora's model that uses co-occurrences as a generic relation, in our work we include relations as labels for the edges in the graph. 
~\citet{Bollegala:AAAI:2018} extended the model proposed by~\citet{Arora:TACL:2016} to capture co-occurrences involving more than two words.
Specifically, they defined the co-occurrence of $k$ unique words in a given context as a $k$-way co-occurrence, where~\citet{Arora:TACL:2016} result could be seen as a special case corresponding to $k=2$.
Moreover, it has been shown that it is possible to learn word embeddings that capture some types of semantic relations such as antonymy and collocation using 3-way co-occurrences more accurately than using 2-way co-occurrences.
However, that model does not explicitly consider the relations between words/entities and uses only a corpus for learning the word embeddings.

\begin{table}[t!]
\centering
\scalebox{0.7}{ 
\begin{tabular}{l l l}
\toprule
KGE method & Score function & Relation \\ 
 &  $f(h, R, t)$ & parameters \\ \midrule
Unstructured\small{~\cite{bordes2014semantic}} & $\norm{\vec{h}-\vec{t}}_{\ell_{1/2}}$ & none \\
Structured\small{~\cite{Bordes:AAAI:2011}} & $\norm{\mat{R}_{1}\vec{h} - \mat{R}_{2}\vec{t}}_{\ell_{1,2}}$ & $\mat{R}_{1}, \mat{R}_{2} \in \R^{d \times d}$ \\
TransE\small{~\cite{Bordes:NIPS:2013}} & $\norm{\vec{h} + \vec{r} - \vec{t}}_{\ell_{1/2}}$ & $\vec{r} \in \R^{d}$\\
DistMult\small{~\cite{Yang:ICLR:2015}} & $\braket{\vec{h}, \vec{r}, \vec{t}}$ & $\vec{r} \in \R^{d}$ \\
RESCAL\small{~\cite{Nickerl:ICML:2011}} & $\vec{h}\T\mat{R}\vec{t}$ & $\mat{R} \in \R^{d \times d}$ \\
ComplEx\small{~\cite{ComplexEmb}}  &  $\braket{\vec{h}, \vec{r}, \bar{\vec{t}}}$ & $\vec{r} \in \mathbb{C}^{d}$ \\
\bottomrule
\end{tabular}
}
\caption{Score functions proposed in selected prior work on KGEs. Entity embeddings $\vec{h}, \vec{t} \in \R^{d}$ are vectors in all models, except in ComplEx where $\vec{h},\vec{t} \in \mathbb{C}^{d}$. Here, $\ell_{1/2}$ denotes either $\ell_{1}$ or $\ell_{2}$ norm of a vector. In ComplEx, $\bar{\vec{t}}$ is the element-wise complex conjugate.}
\label{tbl:scores}
\end{table}

\section{Relational Walk}
\label{sec:randwalk}
Let us consider a KG, $\cD$, where the \emph{knowledge} is represented by relational triples $(h,R,t) \in \cD$. Here, $R$ is a relational predicate with two arguments, where $h$ (\emph{head}) and $t$ (\emph{tail}) entities respectively filling the first and second arguments. In this work, we assume relations to be asymmetric in general (if $(h,R,t) \in \cD$ then it does not necessarily follow that $(t,R,h) \in \cD$). 
% A KG can then be seen as a directed edge-labelled graph where vertices represent entities and an edge connecting two vertices represents a semantic relation that exists between the corresponding entities.
The goal of KGE is to learn embeddings for the relations and entities in the KG such that the entities that participate in similar relations are embedded closely to each other in the entity embedding space, while at the same time relations that hold between similar entities are embedded closely to each other in the relational embedding space. 
We call the learnt entity and relation embeddings collectively as KGEs.
We assume that entities and relations are embedded in the same vector space, allowing us to perform linear algebraic operations using the embeddings in the same vector space.

Following our aforementioned modelling of a knowledge base as a graph, let us consider a random walker who is at a vertex corresponding to some entity $h$.
This entity will have one or more semantic relations with other entities in the KG.
The random walker will uniformly at random pick one of the outgoing edges corresponding to a particular semantic relation $R$, and follow it to land on the entity $t$.
This one-step of the random walk thus \emph{generates} a tuple $(h, R, t)$ in the KG.
The random walker proceeds by using $t$ as the new starting point.
Multiple steps of this random walk trace a single \emph{path} in the KG.

To illustrate a random walk over a KG, let us assume that we are currently at the vertex corresponding to the company entity \textbf{Amazon Inc.}
Possible outgoing edges at \textbf{Amazon Inc.} would correspond to semantic relations such as \textsf{has-ceo}, \textsf{is-headquarted-at}, \textsf{founded-in} etc., where \textbf{Amazon Inc.} is the head entity.
If there are only three such outgoing relations at \textbf{Amazon Inc.}, then the random walker will pick any one of those relations with a probability $1/3$.
For example, by selecting \textsf{has-ceo}, \textsf{is-headquarted-at} or \textsf{founded-in} the random walker would arrive at entities respectively 
\textbf{Jeff Bezos}, \textbf{Seattle} or \textbf{1994}.
Let us assume that the random walker selected the \textsf{has-ceo} relation and landed at \textbf{Jeff Bezos}.
The random walker might subsequently continue its random walk from \textbf{Jeff Bezos} following the relation \textsf{born-in} and transiting to \textbf{New Mexico, US}.
Prior work studying inferences in KGs have successfully used random walk models similar to what we describe here~\cite{gardner-etal-2013-improving,lao-EtAl:2012:EMNLP-CoNLL,lao-etal-2011-random,Lao_2010}.

Let us consider a random walk characterised by a time-dependent \emph{knowledge vector} $\vec{c}_{k}$, where $k$ is the current time step. The knowledge vector represents the knowledge we have about a particular group of entities and relations that express some facts about the world.
For example, when we are talking about \textbf{Amazon Inc.}, we will use the knowledge associated with \textbf{Amazon Inc.} such as its CEO, location of the headquarters, when it was founded etc. 
Therefore, it is intuitive to assume that the entities associated with \textbf{Amazon Inc.} with some set of semantic relations can be generated from this knowledge vector.
%For example, the knowledge that we have about people that are employed by companies can be expressed using entities of classes such as people and organisation, using relations such as \textsf{ceo-of, employed-at, works-for}, etc.
Each entity and relation has time-independent latent representations that capture their correlations with $\vec{c}_{k}$.
For entities $h$ and $t$, we denote their representations by  $d$-dimensional vectors respectively $\vec{h}, \vec{t} \in \R^{d}$.

We assume the task of generating a relational triple $(h,R,t)$ in a given KG to be a two-step process as described next.
First, given the current knowledge vector at time $k$, $\vec{c} = \vec{c}_{k}$ and the relation $R$, we assume that the probability of an entity $h$ satisfying the first argument of $R$ to be given by the loglinear entity production model in~\eqref{eq:h}.
\par\nobreak 
{\small
\vspace{-4mm}
\begin{align}
\label{eq:h}
p(h \mid R, \vec{c}) = \frac{1}{Z_{c}} \exp\left(\vec{h}\T \mat{R}_{1}\vec{c}\right).
\end{align}
}
Here, $\mat{R}_{1} \in \R^{d \times d}$ is a relation-specific orthogonal matrix that evaluates the appropriateness of $h$ for the first argument of $R$.
For example, if $R$ is the \textsf{is-ceo-of} relation, we would require a person as the first argument and a company as the second argument of $R$.
However, note that the role of $\mat{R}_{1}$ extends beyond simply checking the types of the entities that can fill the first argument of a relation.
For our example above, not all people are CEOs and $\mat{R}_{1}$ evaluates the likelihood of a person to be selected as the first argument of the \textsf{ceo-of} relation.
$Z_{c}$ is a normalisation coefficient such that $\sum_{h \in \cV} p(h \mid R,\vec{c}) = 1$, where the vocabulary $\cV$ is the set of all entities in the KG.
%\DB{We need to tell why $\exp$ is used and defend against being called this is also heuristic}

After generating $h$, the state of our random walker changes to $\vec{c}' = \vec{c}_{k+1}$, and we next generate the second argument of $R$ with the probability given by~\eqref{eq:t}.
\par\nobreak 
{\small
\vspace{-4mm}
\begin{align}
\label{eq:t}
p(t \mid R, \vec{c}') = \frac{1}{Z_{c'}} \exp\left(\vec{t}\T \mat{R}_{2}\vec{c}'\right).
\end{align}
}
Here, $\mat{R}_{2} \in \R^{d \times d}$ is a relation-specific orthogonal matrix that evaluates the appropriateness of $t$ as the second argument of $R$.
$Z_{c'}$ is a normalisation coefficient such that $\sum_{t \in \cV} p(t \mid R,\vec{c}') = 1$. 
Following our previous example of \textsf{is-ceo-of} relation, $\mat{R}_{2}$ evaluates the likelihood of an organisation to be a company with a CEO position. Importantly, $\mat{R}_{1}$ and $\mat{R}_{2}$ are representations of the relation $R$ and independent of the entities.
Therefore, we consider $(\mat{R}_{1}$ and $\mat{R}_{2})$ to collectively represent the embedding of $R$. 
Orthogonality of $\mat{R}_{1}, \mat{R}_{2}$ is a requirement for the mathematical proof and also acts as a regularisation constraint to prevent overfitting by restricting the relational embedding space~\cite{tang-etal-2020-orthogonal}.
Intuitively, orthogonality of the relation embedding matrices ensures that the length of the head and tail entity embeddings are not altered during the generation of the tuple.
Orthogonal transformation has been shown to improve the performance of relation representation in prior work. 
For example,~\citet{tang-etal-2020-orthogonal} apply orthogonal transformation that extends RotatE~\cite{sun2019rotate} to model complex relations (e.g., N-to-N). 
For relationships in word embedding space, \citet{ethayarajh-2019-rotate} use orthogonal transformation for analogical reasoning between words.
The performance of hypernymy prediction through orthogonal projections has been improved as shown in~\citet{wang2019family}.

The knowledge vector $\vec{c}_{k}$ performs a \emph{slow} random walk (meaning $\vec{c}_{k+1}$ is obtained from $\vec{c}_{k}$ by adding a small random displacement vector) such that the head and tail entities of a relation are generated under similar knowledge vectors. 
More specifically, we assume that $\norm{\vec{c}_{k} - \vec{c}_{k+1}} \leq \epsilon_{2}$ for some small $\epsilon_{2} > 0$. This is a realistic assumption for generating the two entity arguments in the same relational triple because, if the knowledge vectors were significantly different in the two generation steps, then it is likely that the corresponding relations are also different, which would not be coherent with the above-described generative process.
Moreover, we assume that the knowledge vectors are distributed uniformly in the unit sphere and denote the distribution of knowledge vectors by $\cC$.

To relate KGEs with the connections in the graph, we must estimate the probability that $h$ and $t$ satisfy the relation $R$, $p(h,t \mid R)$, which can be obtained by taking the expectation of $p(h,t \mid R, \vec{c}, \vec{c}')$ w.r.t. $\vec{c}, \vec{c}' \sim \cC$ given by~\eqref{eq:exp1}.
\par\nobreak
{\small
\vspace{-4mm}
\begin{align}
p(h,t \mid R) &= \Ep_{\vec{c},\vec{c}'} \left[ p(h,t  \mid R, \vec{c}, \vec{c}') \right]  \label{eq:exp1} \\
&= \Ep_{\vec{c},\vec{c}'} \left[ p(h \mid R, \vec{c}) p(t \mid R, \vec{c}') \right]  \label{eq:exp2} \\
&=  \Ep_{\vec{c},\vec{c}'} \left[  \frac{\exp\left(\vec{h}\T \mat{R}_{1}\vec{c}\right)}{Z_{c}} \frac{\exp\left(\vec{t}\T \mat{R}_{2}\vec{c}'\right)}{Z_{c'}} \right]. \label{eq:exp3}
\end{align}
}
Here, partition functions are given by
\par\nobreak
{\small
\vspace{-3mm}
\begin{align}
	\label{eq:c} Z_{c} &= \sum_{h \in \cV} \exp\left( \vec{h}\T\mat{R}_{1}\vec{c} \right) \\
	\label{eq:c-prime} Z_{c'} &= \sum_{t \in \cV} \exp\left( \vec{t}\T\mat{R}_{2}\vec{c}' \right)
\end{align}
}
\eqref{eq:exp2} follows from our two-step generative process where the generation of $h$ and $t$ in each step is independent given the relation and the corresponding knowledge vectors.

Computing the expectation in~\eqref{eq:exp3} is generally difficult because of the two partition functions $Z_{c}$ and $Z_{c'}$.
However, Lemma~\ref{lem:concentration}  shows that the partition functions are narrowly distributed around a constant value for all $c$ (or $c'$) values with high probability.

\begin{lemma}[\textbf{Concentration Lemma}]\label{lem:concentration}
  If the entity embedding vectors satisfy the Bayesian prior $\vec{v} = s \hat{\vec{v}}$, where $\vec{\hat{v}}$ is from the spherical Gaussian distribution, and $s$ is a scalar random variable, which is always bounded by a constant $\kappa$, then the entire ensemble of entity embeddings satisfies that:
  \par\nobreak
  {\small
  \vspace{-2mm}
  \begin{align}
    \Pr_{c \sim \cC}[(1-\epsilon_z) Z \leq Z_c \leq (1+\epsilon_z)Z]\geq 1-\delta, \label{eq:2.1-of-Arora}
  \end{align}
  }
  for $\epsilon_z = O (1/\sqrt{n})$, and $\delta = \exp(-\Omega(\log^2 n))$,
  where $n \geq d$ is the number of entities in a given KG and $Z_{c}$ is the partition function for $c$ given by $\sum_{h \in \cV}  \exp \left( \vec{h}\T\mat{R}_{1}\vec{c} \right)$.
  \end{lemma}

Refer to~\autoref{sec:LemmaProof} for the proof of the concentration lemma.
We empirically investigate the relationship between the performance of the KGEs and the degree to which Lemma~\ref{lem:concentration} is satisfied in \autoref{sec:conc-orth}.
Under the conditions required to satisfy Lemma~\ref{lem:concentration}, the following main theorem of this paper holds:
\begin{theorem}\label{th:KB}
  Suppose that the entity embeddings satisfy~\eqref{lem:concentration}.
  Then, we have 
  \par\nobreak
  {\small
  \vspace{-2mm}
  \begin{align}
  \label{eq:joint-prob}
    \log p(h,t \mid R) = \frac{\norm{\mat{R}_{1}\T\vec{h} + \mat{R}_{2}\T\vec{t}}_{2}^2}{2d} - 2\log Z  \pm \epsilon.
  \end{align}
  }
  for $\epsilon = O(1/\sqrt{n}) + \widetilde{O}(1/d)$, where
  \par\nobreak
  {\small
  \vspace{-2mm}
\begin{align}
 \label{eq:Z}
 Z = Z_{c} = Z_{c'}.
\end{align}
}
\end{theorem}

\begin{proof}[Proof sketch:]\label{proof:th}
Let $F$ be the event that both $c$ and $c'$ are within $(1\pm \epsilon_{z})Z$. Then, from Lemma~\ref{lem:concentration}
and the union bound, event $F$ happens with probability at least $1-2\exp(-\Omega(\log^{2}n))$. The R.H.S. of \eqref{eq:exp3} can be split into two parts $T_{1}$ and $T_{2}$ according to whether $F$ happens or not.
{\small
\begin{align}
\label{eq:T12}
&p(h,t \mid R) = \nonumber \\ 
& \underbrace{\Ep_{c,c'}\left[ \frac{\exp\left(\vec{h}\T\mat{R}_{1}\vec{c}\right)}{Z_{c}} \frac{\exp\left(\vec{h}\T\mat{R}_{2}\vec{c'}\right)}{Z_{c'}} \vec{1}_{F} \right]}_{=T_{1}} \nonumber \\
&+ \underbrace{\Ep_{c,c'}\left[ \frac{\exp\left(\vec{h}\T\mat{R}_{1}\vec{c}\right)}{Z_{c}} \frac{\exp\left(\vec{h}\T\mat{R}_{2}\vec{c'}\right)}{Z_{c'}} \vec{1}_{\bar{F}} \right]}_{=T_{2}}.
\end{align}
}
$T_{1}$ can be approximated as given by \eqref{eq:T1:approx}.
{\small
\begin{align}\label{eq:T1:approx}
 T_{1} =  \frac{1\pm\O(\epsilon_{z})}{Z^{2}} \Ep_{c,c'} \left[ \expip{h}{1}{c} \expip{t}{2}{c'}\right] 
\end{align}
}%
On the other hand, $T_{2}$ can be shown to be a constant, independent of $d$, given by \eqref{eq:T2:approx}.
{\small
\begin{align}\label{eq:T2:approx}
 |T_{2}| = \exp(-\Omega(\log^{1.8}n))
\end{align}
}%
The vocabulary size $n$ of real-world KGs is typically over $10^{5}$, for which $T_{2}$ becomes negligibly small.
Therefore, it suffices to consider only $T_{1}$. Because of the slowness of the random walk we have $\vec{c} \approx \vec{c}'$.

Using the law of total expectation we can write $T_{1}$ as follows:
{\small
\begin{align}
 T_{1}  &= \frac{1\pm\O(\epsilon_{z})}{Z^{2}} \Ep_{c} \left[ \expip{h}{1}{c} \Ep_{c'|c} \left[ \expip{t}{2}{c'} \right] \right] \nonumber \\
 &=\frac{1\pm\O(\epsilon_{z})}{Z^{2}} \Ep_{c} \left[ \expip{h}{1}{c} A(c) \right] \label{eq:T1:final}
\end{align}
}
where $A(c) \coloneqq \Ep_{c'|c}\left[ \expip{t}{2}{c'}\right]$.
Doing some further evaluations we show that 
{\small
\begin{align}
\label{eq:A}
A(c) = (1 \pm \epsilon_{2}) \expip{t}{2}{c}
\end{align}
}%
Plugging \eqref{eq:A} back in \eqref{eq:T1:final} provides the claim of the theorem.
Detailed proof is shown in~\autoref{sec:TheoremProof}.
\end{proof}
\vspace{-2mm}
The relationship given by \eqref{eq:joint-prob} indicates that  head and tail entity embeddings are first transformed respectively by $\mat{R}_{1}\T$ and $\mat{R}_{2}\T$, and the squared $\ell_{2}$ norm of the sum of the transformed vectors is proportional to the probability $p(h,t \mid R)$.
%proof for theorem 1
\section{Learning KG Embeddings}
\label{sec:learn}

In this section, we derive a training objective from \autoref{th:KB} that we can then optimise to learn KGEs.
The goal is to empirically validate the theoretical result by evaluating the learnt KGEs.
KGs represent information about relations between two entities in the form of \emph{relational triples}.
The joint probability $p(h,R,t)$ given by \autoref{th:KB} is useful for determining whether a relation $R$ exists between two given entities $h$ and $t$. For example, if we know that with a high probability that $R$ holds between $h$ and $t$, then we can append $(h,R,t)$ to the KG.
The task of expanding KGs by predicting missing links between entities or relations is known as the \emph{link prediction} problem~\cite{ComplexEmb}.
In particular, if we can automatically append such previously unknown knowledge to the KG, we can expand the KG and address the knowledge acquisition bottleneck.
% add few citations for the previous para.

To derive a criteria for determining whether a link must be predicted among entities and relations, let us consider a relational triple $(h,R,t) \in \cD$ that exists in a given KG $\cD$. We call such relational triples as \emph{positive} triples because from the assumption it is known that $R$ holds between $h$ and $t$. On the other hand, consider a \emph{negative} relational triple $(h', R, t') \in \cD$ formed by, for example, randomly perturbing a positive triple. A popular technique for generating such (pseudo) negative triples is to replace $h$ or $t$ with a randomly selected different instance of the same entity type. As an alternative for random perturbation,  \newcite{KBGAN} proposed a method for generating negative instances using adversarial learning. Here, we are not concerned about the actual method used for generating the negative triples but assume a set of negative triples, $\bar{\cD}$, generated using some method, to be given.

Given a positive triple $(h,R,t) \in \cD$ and a negative triple $(h',R,t') \in \bar{\cD}$, we would like to learn KGEs such that a higher probability is assigned to $(h,R,t)$ than that assigned to $(h',R,t')$. We can formalise this requirement using the likelihood ratio given by \eqref{eq:llr}.
\par\nobreak
{\small
\vspace{-4mm}
\begin{align}
 \label{eq:llr}
 \frac{p(h,R,t)}{p(h',R,t')} \geq \eta
\end{align}
}
Here, $\eta > 1$ is a threshold that determines how higher we would like to set the probabilities for the positive triples compares to that of the negative triples.

By taking the logarithm of both sides in \eqref{eq:llr} we obtain
{\small
\begin{align}
 \log p(h,R,t) - \log p(h', R, t') &\geq \log\eta \nonumber \\
 \log\eta + \log p(h', R, t') -  \log p(h,R,t) & \leq 0 \label{eq:llr2}
\end{align}
}
If a positive triple $(h,R,t)$ is correctly assigned a higher probability than a negative triple $p(h',R,t')$, then the left hand side of \eqref{eq:llr2} will be negative, indicating that there is no \emph{loss} incurred during this classification task. 
Therefore, we can re-write \eqref{eq:llr2} to obtain the \emph{marginal loss}~\cite{Bordes:NIPS:2013,Bordes:AAAI:2011}, $L(\cD, \bar{\cD})$, a popular choice as a learning objective in prior work in KGE, as shown in \eqref{eq:loss}.
\par\nobreak
{\small
\begin{align}
\label{eq:loss}
& L(\cD, \bar{\cD}) = \nonumber \\ 
& \mkern-18mu \sum_{\substack{(h,R,t)\in \cD \\ (h',R,t') \in \bar{\cD}}} \mkern-24mu \max \left( 0,  \log\eta + \log p(h', R, t') -  \log p(h,R,t) \right) \nonumber \\
 =& \max \Big( 0, 2d\log\eta + \norm{\mat{R}_{1}\T\vec{h}' + \mat{R}_{2}\T\vec{t}'}_{2}^2 \nonumber \\
 &- \norm{\mat{R}_{1}\T\vec{h} + \mat{R}_{2}\T\vec{t}}_{2}^2 \Big) 
\end{align}
}
We can assume $2d\log\eta$ to be the \emph{margin} for the constraint violation. 

\autoref{th:KB} requires $\mat{R}_{1}$ and $\mat{R}_{2}$ to be orthogonal. 
To reflect this requirement, we add two $\ell_{2}$ regularisation terms $\norm{\mat{R}_{1}\T\mat{R}_{1} - \mat{I}}_{2}^{2}$
and $\norm{\mat{R}_{2}\T\mat{R}_{2} - \mat{I}}_{2}^{2}$ respectively with regularisation coefficients $\lambda_{1}$ and $\lambda_{2}$ to the objective function given by \eqref{eq:loss}.
In our experiments, we compute the gradients \eqref{eq:loss} w.r.t. each of the parameters
$\vec{h}$, $\vec{t}$, $\vec{R}_{1}$ and $\vec{R}_{2}$ and use stochastic gradient descent (SGD) for optimisation.
Considering that negative triples are generated via random perturbation, it is important to consider multiple negative triples during training to better estimate the classification boundary.
This approach can be easily extended to learn from multiple negative triples as shown in~\autoref{sec:MultipleNeg}.

%The next section will empirically assess entity and relation embeddings of a given KG that are learnt under the scoring function derived by the \textbf{RelWalk} model.  

\section{Empirical validation}

\begin{table*}[!htb]
\small
    \begin{subtable}{.6\linewidth}
      \centering
        % \caption{}
        \scalebox{0.8}{
        \begin{tabular}{l c c c c c  c c c c c} \toprule
			 & \multicolumn{5}{c}{FB15K237} & \multicolumn{5}{c}{WN18RR} \\ 
			 \cmidrule(lr){2-6}  \cmidrule(lr){7-11} \\
			Method& MRR& MR & H@1&H@3&H@10 & MRR& MR & H@1&H@3&H@10\\ \midrule
			TransE$^\bullet$ &0.294&347&-&-&0.465&0.226&3384&-&-&0.50 \\
			TransD$^\triangleleft$ &0.280&-&-&-&0.453&-&-&-&-&0.43 \\
			DistMult$^\star$ &0.241&254&0.155&0.263&0.419&0.430&5110&0.39&0.44&0.49\\
			ComplEx$^\star$ &0.247&339&0.158&0.275&0.428&0.440&5261&0.41&0.46&0.51\\
			ConvE\scriptsize{~\cite{dettmers2017convolutional}} & 0.325&244&0.237& \textbf{0.356}&0.501& 0.430 &4187&0.40&0.44&0.52\\
			CP-N3\scriptsize{~\cite{Lacroix:2018aa}} &  \textbf{0.360} & - & - & - & \textbf{0.540} & \textbf{0.470} &  - & - & - & \textbf{0.54} \\
			\hline
			RelWalk & 0.329 &\textbf{105}&\textbf{0.243}& 0.354 & 0.502 &0.451&\textbf{3232}		&\textbf{0.42}&\textbf{0.47}& 0.51\\
			\bottomrule
	\end{tabular}}
    \end{subtable}%
    \begin{subtable}{.56\linewidth}
      \centering
        % \caption{}
        \scalebox{0.8}{
        \begin{tabular}{l c } \toprule
			Method   &	Accuracy \\\midrule
			Structured$^\diamond$ & 75.2 \\
			TransE$^\diamond$  & 81.5 \\
			TransR\scriptsize{~\cite{Lin:AAAI:2015}} & 82.5 \\
			TransG\scriptsize{~\cite{xiao-huang-zhu:2016:P16-1}}  & 87.3 \\
			NTN\scriptsize{~\cite{Socher:NIPS:2013b}}  & 87.2 \\
			\hline
			RelWalk & \textbf{88.6} \\
			\bottomrule
			\end{tabular}}
    \end{subtable} 
    \caption{Results of link prediction (left) and triple classification on FB13 (right). Results marked with $[\star]$ are taken from~\cite{dettmers2017convolutional}, $[\bullet]$ from~\cite{Nguyen:NAACL:2016},$[\triangleleft]$ from~\cite{KBGAN} and $[\diamond]$ from~\cite{Wang:AAAI:2014}. All other results for the baselines are taken from their original papers.}
    \label{tab:EmpiricalRes}
\end{table*}

To empirically evaluate the theoretical result stated in \autoref{th:KB}, we learn KGEs (denoted by \textbf{RelWalk}) by minimising the marginal loss objective derived in \autoref{sec:learn}. We use the FB15k237, FB13 (subsets of \emph{Freebase}) and WN18RR (a subset of \emph{WordNet}) datasets, which are standard benchmarks for KGE. 
% tell the issues due to symmetry
We use the standard training, validation and test splits. Statistics about the datasets and training details are  in~\autoref{sec:TrainDetails}. 
\textbf{RelWalk} is implemented in the open-source toolkit OpenKE~\cite{han2018openke} and the code and learnt KGEs will are publicly available\footnote{\url{https://github.com/LivNLP/Relational-Walk-for-Knowledge-Graphs}}.

\begin{table}[h!]
 \small
\centering
 \scalebox{0.85}{
 \begin{tabular}{p{26mm} p{5mm} p{6mm} p{6mm} p{5mm} c}\toprule
 Relation & H@10 & $\nu_{R}$ & $\sigma_{c}$ & $\sigma_{c'}$ &  $\sqrt{\sigma_{c}^2+\sigma_{c'}^2}$  \\ \midrule
 hypernym  & 0.188 & 3.249 & 68.89  & 64.41  &  94.31  \\
 derivational  & 0.955 & 1.690 & 63.44  & 65.33  &  91.07 \\
 instance\_hypernym  & 0.541 & 0.362 & 63.11  & 64.56  & 90.28  \\
 also\_see  & 0.670 & 0.234 &  70.76 &  61.51 & 93.76  \\
 member\_meronym  & 0.281 & 4.389 & 63.78  & 66.09  & 91.84  \\
 synset\_domain\_topic  & 0.513 & 0.727 & 65.66  &  65.48 &92.73   \\
 has\_part  & 0.247 & 0.548 & 66.21  &  66.50 &  93.84 \\
 domain\_usage  & 0.688 & 0.045 &  65.24 & 63.16  &90.81   \\
 domain\_region  & 0.442 & 0.065 & 67.53  & 66.31 &94.64   \\
 verb\_group  & 0.974 & 0.038 & 64.22  & 63.19  & 90.09  \\
 similar\_to  & 1.000 & 0.111 & 63.67  & 63.96 & 90.25  \\
 \midrule
 Correlations  & & $-0.51$ &  $-0.39$ & $-0.49$ &  $-0.70$ \\
 \bottomrule
 \end{tabular}}
 \caption{Empirical analysis of the concentration of the partitioning functions and the orthogonality of the relation embeddings, and their Pearson correlation coefficients against H@10 for the relations in WN18RR.}
 \label{tbl:co}
\end{table}

We conduct two evaluation tasks: \emph{link prediction} (predict the missing head or tail entity in a given triple $(h, R, ?)$ or $(?,R,t)$)~\cite{Bordes:AAAI:2011} and \emph{triple classification} (predict whether a relation $R$ holds between $h$ and $t$ in a given triple $(h,R,t)$)~\cite{Socher:NIPS:2013b}. 
We evaluate the performance in the link prediction task using mean reciprocal rank (\textbf{MRR}), mean rank (\textbf{MR}) (the average of the rank assigned to the original head or tail entity in a corrupted triple) and hits at ranks 1, 3 and 10 (\textbf{H@1, 3, 10}),
  whereas in the triple classification task we use \textbf{accuracy} (percentage of the correctly classified test triples). 
We only report scores under the \emph{filtered} setting~\cite{Bordes:NIPS:2013}, which removes all triples appeared in training, validating and testing sets from candidate triples before obtaining the rank of the ground truth triple.
In link prediction, we consider all entities that appear in the corresponding argument in the entire knowledge graph as candidates.

In~\autoref{tab:EmpiricalRes} we compare the KGEs learnt by \textbf{RelWalk} against prior work using the published results.
 For triple classification, \textbf{RelWalk} reports the best performance on FB13, outperforming all methods compared.
%Considering that both TransG and \textbf{RelWalk} are generative models, it would be interesting to further investigate generative approaches for KGE in the future.
For the link prediction results as shown in~\autoref{tab:EmpiricalRes}, we see that \textbf{RelWalk} obtains competitive performance on both WN18RR and FB15K237 under all evaluation measures. 
In particular, it is outperformed by the KGE method proposed by~\citet{Lacroix:2018aa} (\textbf{CP-N3}), which uses nuclear 3-norm regularisers with canonical tensor decomposition.
Interestingly, the improvement against structured embeddings (SE) is consistent and interesting because the scoring function of SE closely resembles that of RelWalk as we can redefine $\mat{R}_{2}$ with the negative sign. However, SE learns KGEs that \emph{minimise} the $\ell_{1,2}$ norm whereas according to \eqref{eq:joint-prob} we must \emph{maximise} the probability for relational triples in a knowledge graph. 
WN18RR excludes triples from WN18 that are simply inverted between train and test partitions~\cite{Toutanova:2015,1707.01476}, making it a difficult dataset for link prediction using simple memorisation heuristics.
 \textbf{RelWalk}'s consistent good performance on both versions of this dataset shows that it is considering the global structure in the KG when learning KGEs. 
 
 We note that our goal in this paper is \emph{not} to claim SoTA for KGE but to provide a theoretical understanding with empirical validation.
To this end, the experimental results support our theoretical claim and emphasise the importance of theoretically motivating the KGE scoring function design process. 

\subsection{Orthogonality and Concentration}
\label{sec:conc-orth}

Our theoretical analysis depends on two main assumptions: (a) concentration of the partition function $Z_{c}$ (Lemma~\ref{lem:concentration}), and (b) the orthogonality of the relation embedding matrices $\mat{R}_{1}, \mat{R}_{2}$. 
In this section, we empirically study the relationship between these assumptions and the performance of RelWalk.
 \begin{figure*}[t!]
 \centering
 \includegraphics[width=5.8in]{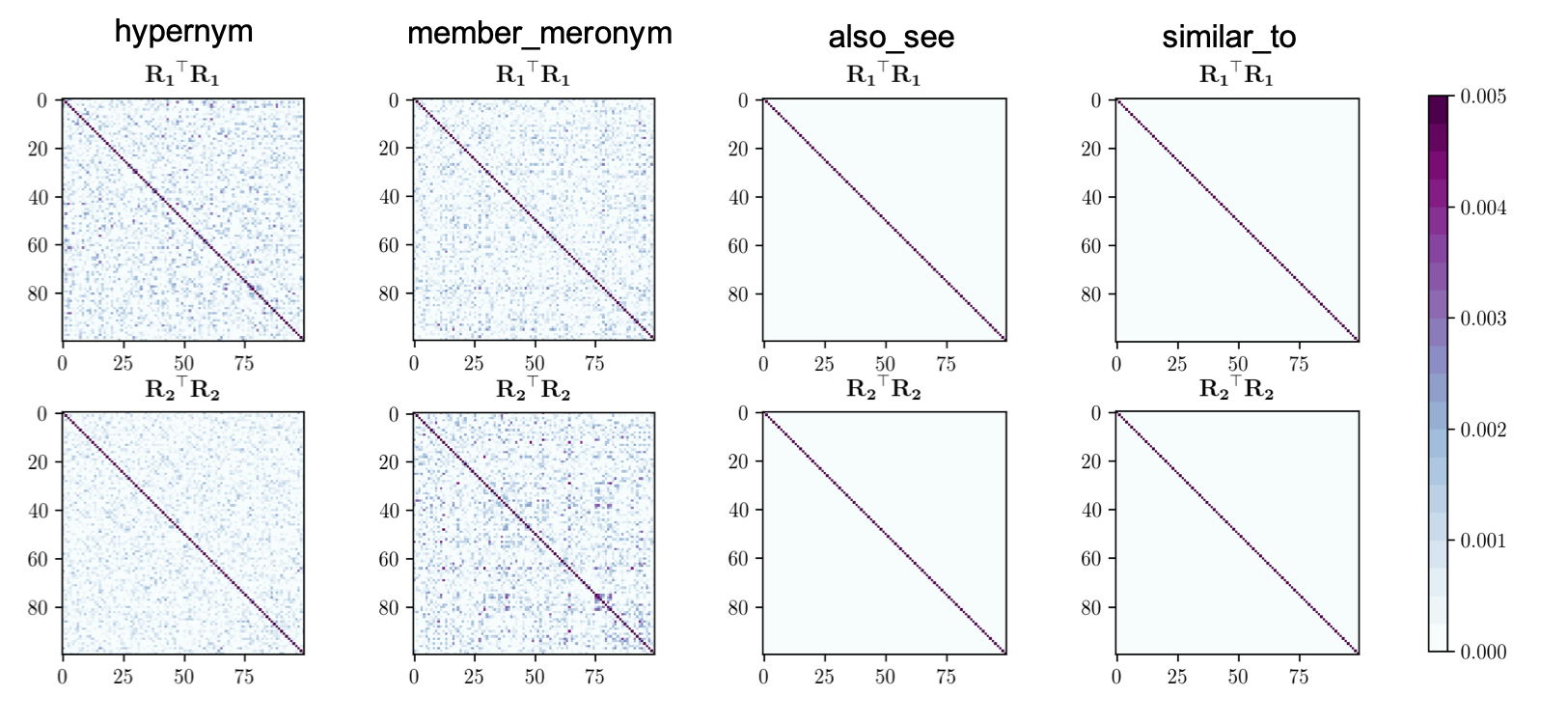}
 \caption{Heatmap visualisation of the orthogonality in different relation embeddings from the WN18RR.}
 \label{fig:orth}
\end{figure*}

Given $\mat{R}_{1}$ and $\mat{R}_{2}$ learnt by RelWalk for a particular $R$, we can measure the degree to which the orthogonality, $\nu_{R}$, is satisfied by the sum of the non-diagonal elements \eqref{eq:orthogonality}.
\par\nobreak
{\small
\begin{align}
 \label{eq:orthogonality}
 \nu_{R} = \sum_{i \neq j} |\mat{R}_{1}\T\mat{R}_{1}|_{ij} +  |\mat{R}_{2}\T\mat{R}_{2}|_{ij}
\end{align}
}
If a matrix $\mat{A}$ is orthogonal, then the non-diagonal elements of the inner-product $\mat{A}\T\mat{A}$ will contain zeros.
Therefore, the smaller the $\nu_{R}$ values, more orthogonal the relation embeddings will be.
We measure $\nu_{R}$ values for the 11 relation types in the WN18RR dataset as shown in \autoref{tbl:co}.
From \autoref{tbl:co} we see that $\nu_{R}$ values are indeed small for different relation types indicating that the orthogonality requirement is satisfied as expected.
Interestingly, a moderately high (-0.515) negative Pearson correlation between H@10 and $\nu_{R}$ shows that orthogonality correlates with the better the performance.

To visualise how the orthogonality affects different relation types, we plot the elements in $\mat{R}_{1}\T\mat{R}_{1}$ and $\mat{R}_{2}\T\mat{R}_{2}$ for four relations in the WN18RR dataset in \autoref{fig:orth} for $100 \times 100$ dimensional relational embeddings.
For the two relations \textsf{also\_see} and \textsf{similar\_to} we see that the corresponding inner-products are sparse except in the main diagonal,
compared to that in \textsf{hypernym} and \textsf{member\_meronym} relations. 
On the other hand, according to \autoref{tbl:co} the H@10 values for  \textsf{also\_see} and \textsf{similar\_to} are higher than that for \textsf{hypernym} and \textsf{member\_meronym} as implied by the negative correlation.

To test for the concentration of the partition function, for a relation $R$ we compute $Z_{c}$ and $Z_{c'}$ values using respectively \eqref{eq:c} and \eqref{eq:c-prime} over a set of randomly sampled 10000 head or tail entities. We compute the standard deviations $\sigma_{c}$ and $\sigma_{c'}$ respectively for the distributions of $Z_{c}$ and $Z_{c'}$ and their geometric means as shown in \autoref{tbl:co}. 
We observed a Gaussian-like distributions for the partition functions for different relations for which smaller standard deviations indicate stronger concentration around the mean. 
Interestingly, from \autoref{tbl:co} we see a negative correlation between H@10 and the standard deviations indicating that the performance of RelWalk depends on the validity of the concentration assumption.

\subsection{Compression of Embeddings}
 \begin{figure}[t!]
 \centering
 \includegraphics[width=2.8in]{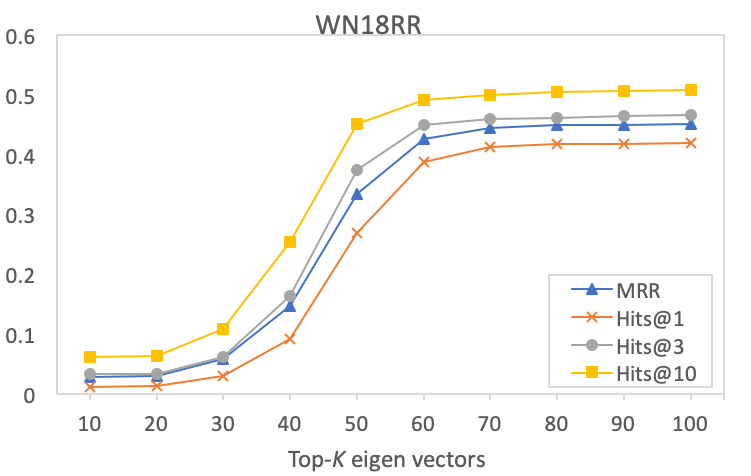}
 \caption{Results for the approximated relation embeddings for link prediction on WN18RR.}
 \label{fig:LowRank}
 \end{figure}
 
%Almost all of the previously proposed KGE methods learn vector embeddings to represent head and tail entities.
%In this regard, the entity embeddings learnt by \textbf{RelWalk} adhere to this standard practice.
%However, as already discussed in \autoref{sec:related}, when it comes to relation embeddings, there exists a vast disparity among KGE methods such as vectors~\cite{Bordes:AAAI:2011,Wang:AAAI:2014,ji-EtAl:2015:ACL-IJCNLP,xiao-huang-zhu:2016:P16-1,Lin:AAAI:2015,Yoon:2016aa,Yang:ICLR:2015,Nickel:AAAI:2016,Lin:AAAI:2015}, matrices~\cite{Bordes:AAAI:2011,Ji:AAAI:2016,Nguyen:NAACL:2016,Nickerl:ICML:2011} or tensors~\cite{Socher:NIPS:2013b}.
%Vector embeddings are attractive over matrix or tensor embeddings from a storage and computation point-of-view when representing large number of relations. 

To reduce the amount of memory required for KGEs, especially with a large KG, compressing KGEs has been studied recently~\cite{sachan-2020-knowledge}. 
\textbf{RelWalk} uses (orthogonal) matrices to represent relations,  which require more parameters compared to a vector representation of the same dimensionality of a relation.
Prior work studying lower-rank decomposition of KGEs have shown that, although linear embeddings of graphs can require prohibitively large dimensionality to model certain types of relations~\cite{Nickel:2014aa} (e.g. \textsf{sameAs}), nonlinear embeddings can mitigate this problem~\cite{Bouchard:2015}.
In this section, we propose memory-efficient low-rank approximations to the \textbf{RelWalk} embeddings. 

From the definition of orthogonality it follows that the relation embeddings $\mat{R}_{1}, \mat{R}_{2} \in \R^{d \times d}$ learnt by \textbf{RelWalk} for a particular relation $R$ means that $\mat{R}_{1}, \mat{R}_{2}$  are both full-rank and cannot be factorised as the product of two lower rank matrices.
This prevents us from directly applying matrix decomposition methods such as non-negative matrix factorisation  on the learnt relation embeddings to obtain low-rank approximations. 
Therefore, we subtract the identity matrix $\mat{I} \in \R^{n \times n}$ from the relation embedding $\mat{R} (\in \{\mat{R}_{1}, \mat{R}_{2}\})$  and factorise the remainder $\mat{R}' \in \R^{n \times n}$ as the product of two low-rank matrices using the eigendecomposition of $\mat{R}'$ as given by \eqref{eq:decomp}.
\par\nobreak
{\small
\vspace{-4mm}
\begin{align}
\label{eq:decomp}
\mat{R} &= \mat{I} + \mat{R}' \nonumber \\
	      &= \mat{I} + \mat{U}_{R} \mat{D} \mat{U}_{R}\T \nonumber \\
	      &\approx \mat{I} + \sum_{k = 1}^{K} \mat{D}_{(k,k)}{\mat{U}_{R}}_{(k,:)} {\mat{U}_{R}}_{(:,k)}
\end{align}
}
Here, $\mat{U}$ is the matrix formed by arranging the eigenvectors of $\mat{R}'$ as columns, and  $\mat{D}$ is a diagonal matrix containing the eigenvalues of $\mat{R}'$ in the descending order.
We can then use the largest $K \leq d$ eigenvalues and corresponding eigenvectors to obtain a rank-$K$ approximation in the sense of minimum Frobenius distance between $\mat{R}'$ and its rank-$K$ approximation.
In the case we use $K$ factors in the approximation, we must store $dK$ real numbers corresponding to the  $d$-dimensional eigen vectors per each of the $K$ components  as opposed to $d^{2}$ real numbers in $\mat{R}$.
The compression ratio in this case becomes $dK / d^{2} = K/d$.
When $K \ll d$, this results in a significant compression.

To empirically evaluate the trade-off between the number of singular vectors used in the compression and the accuracy of the learnt relation embeddings, we use the approximated relation embeddings for link prediction on WN18RR as shown in Figure~\ref{fig:LowRank} (similar trend was observed for FB15K237).
We use $d = 100$ dimensional relation embeddings learnt by \textbf{RelWalk} and approximate using top-$K$ eigenvectors.
From~\autoref{fig:LowRank} we see for $K > 60$ components the performance saturates in both datasets.
On the other hand, we need at least $K = 30$ components to get any meaningful accuracy for link prediction on these two datasets.
With $K=60$ and $d=100$ this approximation results in an $60\%$ compression ratio. 
 
\section{Conclusion}
We proposed \textbf{RelWalk}, a generative model of KGE and derived a theoretical relationship between the probability of a triple consisting of head, tail entities and the relation that exists between those two entities, and the embedding vectors for the two entities and embeddings matrices for the relation. 
In RelWalk, we represented entities by vectors and relations by matrices.
We then proposed a learning objective based on the theoretical relationship we derived to learn entity and relation embeddings from a given knowledge graph.
Experimental results on a link prediction and a triple classification tasks  show that  \textbf{RelWalk} 
outperforms several previously proposed KGE learning methods.
The key assumptions of \textbf{RelWalk} are validated by empirically analysing the relationship between such assumptions and the performance of the learnt embeddings from a KG.
Moreover, we studied the compressibility of the learnt relation embeddings and discovered that using only $60\%$ of the components, we can approximate the relation embeddings without any significant loss in performance.
\section*{Acknowledgement}
Yuichi Yoshida is supported by JSPS KAKENHI Grant Number JP18H05291.

\bibliography{relwalk}
\bibliographystyle{acl_natbib}

\appendix

\section{Proof of the Concentration Lemma}
\label{sec:LemmaProof}

To prove the concentration lemma, we show that the mean $\Ep_{\vec{h}}[Z_{c}]$ of $Z_{c}$ is concentrated around a constant for all knowledge vectors $\vec{c}$ and its variance is bounded.
If $\mat{P}$ is an orthogonal matrix and $\vec{x}$ is a vector, then $\norm{\mat{P}\T\vec{x}}_{2}^{2} = (\mat{P}\T\vec{x})\T(\mat{P}\T\vec{x}) = \vec{x}\T\mat{P}\mat{P}\T\vec{x} = \norm{\vec{x}}_{2}^{2}$, because $\mat{P}\T\mat{P} = \mat{I}$.
Therefore, from~\eqref{eq:c} and the orthogonality of the relational embeddings, we see that $\mat{R}_{1}\vec{c}$ is a simple rotation of $\vec{c}$ and does not alter the length of $\vec{c}$.
We represent $\vec{h} =  s_{h} \hat{\vec{h}}$, where $s_{h} = \norm{\vec{h}}$ and $\hat{\vec{h}}$ is a unit vector (i.e. $\norm{\hat{\vec{h}}}_{2} = 1$) distributed on the spherical Gaussian with zero mean and unit covariance matrix $\mat{I}_{d} \in \R^{d \times d}$. Let $s$ be a random variable that has the same distribution as $s_{h}$.
Moreover, let us assume that $s$ is upper bounded by a constant $\kappa$ such that $s \leq \kappa$.
From the assumption of the knowledge vector $\vec{c}$, it is on the unit sphere as well, which is then rotated by $\mat{R}_{1}$.

We can write the partition function using the inner-product between two vectors $\vec{h}$ and $\mat{R}_{1}\vec{c}$,
 $Z_{c} = \sum_{h \in \cV} \exp \left( \vec{h}\T (\mat{R}_{1} \vec{c}) \right)$.
Arora et al.\cite{Arora:TACL:2016} showed that (Lemma 2.1 in their paper) the expectation of a partition function of this form can be approximated as follows:
\begin{align}
\Ep_{\vec{h}}[Z_{c}] &= n \Ep_{\vec{h}}[\exp \left(\vec{h}\T\mat{R}_{1}\vec{c} \right)]  \label{eq:z1} \\
&\geq n \Ep_{\vec{h}}[1 + \vec{h}\T\mat{R}_{1}\vec{c}] = n. \label{eq:z2}
\end{align}
where $n = |\cV|$ is the number of entities in the vocabulary.
\eqref{eq:z1} follows from the expectation of a sum and the independence of $\vec{h}$ and $\mat{R}_{1}$ from $\vec{c}$.
The inequality of~\eqref{eq:z2} is obtained by applying the Taylor expansion of the exponential series and the final equality is due to the symmetry of the spherical Gaussian.
From the law of total expectation, we can write
\begin{align}
 \Ep_{\vec{h}}[Z_{c}] &= n \Ep_{\vec{h}}[\exp \left(\vec{h}\T\mat{R}_{1}\vec{c} \right)]  \nonumber \\ 
 &= n \Ep_{s_{h}}\left[ \Ep_{x \mid s_{h}} \left[  \exp \left(\vec{h}\T\mat{R}_{1}\vec{c} \right) \mid s_{h} \right] \right] \label{eq:totalexp} .
 \end{align}
where, $x = \vec{h}\T\mat{R}_{1}\vec{c}$.
 Note that conditioned on $s_{h}$, $\vec{h}$ is a Gaussian random variable with variance $\sigma^{2} = s_{h}^{2}$.
 Therefore, conditioned on $s_{h}$, $x$ is a random variable with variance $\sigma^{2} = \sigma_{h}^{2}$.
 Using this distribution, we can evaluate $\Ep_{x \mid s_{h}} \left[  \exp \left(\vec{h}\T\mat{R}_{1}\vec{c} \right) \right]$ as follows:
 \begin{align}
 \label{eq:integration}
& \Ep_{x \mid s_{h}} \left[  \exp \left(\vec{h}\T\mat{R}_{1}\vec{c} \right) \mid s_{h} \right]  \nonumber \\
&= \int_{x} \frac{1}{\sqrt{2 \pi \sigma^{2}}} \exp \left( -\frac{x^{2}}{2 \sigma^{2}} \right) \exp(x) dx \nonumber \\
 &= \int_{x} \frac{1}{\sqrt{2 \pi \sigma^{2}}} \exp \left( - \frac{{(x-\sigma^{2})}^{2}}{2\sigma^{2}} + \sigma^{2}/2 \right) dx \nonumber \\
&= \exp(\sigma^{2}/2).
\end{align}
Therefore, it follows that
\begin{align}
\Ep_{\vec{h}}[Z_{c}] &= n \Ep_{s_{h}}[\exp(\sigma^{2}/2)] \nonumber \\
&= n \Ep_{s_{h	}}[\exp(s_{h}^{2}/2)] = n \exp(s^{2}/2),
\end{align}
where $s$ is the variance of the $\ell_{2}$ norms of the entity embeddings. Because the set of entities is given and fixed, both $n$ and $\sigma$ are
constants, proving that $\Ep_{\vec{h}}[Z_{c}]$ does not depend on $c$.

% show that variance can also be bounded

Next, we calculate the variance $\Var_{\vec{h}}[Z_{c}]$ as follows:
\begin{align}
\label{eq:var1}
 \Var_{\vec{h}}[Z_{c}] &= \sum_{h} \Var_{\vec{h}}[\expip{h}{1}{c}] \nonumber \\
 &\leq n\Ep_{\vec{h}} \left[\exp \left( 2\vec{h}\T\mat{R}_{1}\vec{c} \right) \right] \nonumber\\
 &= n \Ep_{s_{h}} \left[ \Ep_{x \mid s_{h}} \left[ \exp \left( 2 \vec{h}\T\mat{R}_{1}\vec{c} \right) \mid s_{h} \right] \right].
\end{align}
Because $2  \vec{h}\T\mat{R}_{1}\vec{c}$ is a Gaussian random variable with variance $4\sigma^{2} = 4 s_{h}^{2}$
from a similar calculation as in~\eqref{eq:integration} we obtain,
\begin{align}
\label{eq:var2}
\Ep_{x \mid s_{h}} \left[ \exp \left( 2 \vec{h}\T\mat{R}_{1}\vec{c} \right) \mid s_{h} \right] = \exp(2\sigma^{2}).
\end{align}
By substituting~\eqref{eq:var2} in~\eqref{eq:var1} we have that
\par\nobreak
{\small
\begin{align}
\Var_{\vec{h}}[Z_{c}] &\leq n \Ep_{s_{h}} \left[ \exp \left( 2\sigma^{2} \right) \right] = n\Ep_{s_{h}} \left[ \exp (2 s^{2}) \right] \leq \Lambda n
\end{align}
}%
for $\Lambda = \exp(8\kappa^{2})$ a constant bounding $s \leq \kappa$ as stated. 
From above, we have bounded both the mean and variance of the partition function by constants that are independent of the knowledge vector.
Note that neither $\expip{h}{1}{c}$ nor $\expip{t}{2}{c'}$ are sub-Gaussian nor sub-exponential. Therefore, standard concentration bounds derived
for sub-Gaussian or sub-exponential random variables cannot be used in our analysis.
However, the argument given in Appendix~A.1 in~\cite{Arora:arxiv:2016} for a partition function with bounded mean and variance can be
directly applied to $Z_{c}$ in our case, which completes the proof of the concentration lemma. 
From the symmetry between $h$ and $t$, 
the concentration Lemma is also applies for the partition function $Z_{c'}=\sum_{t \in \cV} \left( \vec{t}\T\mat{R}_{2}\vec{c'} \right)$.

\section{Proof of RelWalk Theorem}
\label{sec:TheoremProof}

Let us consider the probabilistic event that $(1-\epsilon_{z})Z \leq Z_{c} \leq (1+\epsilon_{z})Z$ to be $F_{c}$ and
$(1-\epsilon_{z})Z \leq Z_{c'} \leq (1+\epsilon_{z})Z$ to be $F_{c'}$. From Lemma~\autoref{lem:concentration} we have $\Pr[F_{c}] \geq 1-\delta$. Then from the union bound we have,
\begin{align}
\Pr[\bar{F_{c}} \cup \bar{F}_{c'}] &\leq \Pr[\bar{F_{c}}] + \Pr[\bar{F}_{c'}] \nonumber \\
&= 1-\Pr[F_{c}] + 1 - \Pr[F_{c'}] \nonumber \\
&= 2\delta.
\end{align}
where $\bar{F}$ is the complement of event $F$.
Moreover, let $F$ be the probabilistic event that both $F_{c}$ and $F_{c'}$ being True.
Then from $\Pr[F] = 1 -\Pr[\bar{F_{c}} \cup \bar{F}_{c'}]$ we have, $\Pr[F] \geq 1-2\exp\left(-\Omega\left(\log^{2}n\right)\right)$.
The R.H.S. of \eqref{eq:exp3} can be split into two parts $T_{1}$ and $T_{2}$ according to whether $F$ happens or not.
\par\nobreak
{\small
\begin{align}
\label{eq:T12}
p(h,t \mid R) = &\underbrace{\Ep_{\vec{c},\vec{c}'}\left[ \frac{\exp\left(\vec{h}\T\mat{R}_{1}\vec{c}\right)}{Z_{c}} \frac{\exp\left(\vec{t}\T\mat{R}_{2}\vec{c'}\right)}{Z_{c'}} \vec{1}_{F} \right]}_{T_{1}} \nonumber \\ 
&+  \underbrace{\Ep_{\vec{c},\vec{c}'}\left[ \frac{\exp\left(\vec{h}\T\mat{R}_{1}\vec{c}\right)}{Z_{c}} \frac{\exp\left(\vec{t}\T\mat{R}_{2}\vec{c'}\right)}{Z_{c'}} \vec{1}_{\bar{F}} \right]}_{T_{2}}.
\end{align}
}%
Here, $\vec{1}_{F}$ and $\vec{1}_{\bar{F}}$ are indicator functions of the events $F$ and $\bar{F}$ given as follows:
\begin{align}
\vec{1}_{F} = \begin{cases} 1 & \text{if $F$ is True}, \\ 0 & \text{otherwise}, \end{cases}
\end{align} 
\begin{align}
\vec{1}_{\bar{F}} = \begin{cases} 0 & \text{if $F$ is True}, \\ 1 & \text{otherwise}. \end{cases}
\end{align}

\noindent\underline{Let us first show that $T_{2}$ is negligibly small}.\\
For two real integrable functions $\psi_{1}(x)$ and $\psi_{2}(x)$ in $[a,b]$, the Cauchy-Schwarz's inequality states that
{\small
\begin{align}
\label{eq:cs}
{\left[\int_{a}^{b} \psi_{1}(x)\psi_{2}(x) dx \right]}^{2} \leq \int_{a}^{b} {\left[\psi_{1}(x)\right]}^{2}dx  \int_{a}^{b} {\left[\psi_{2}(x)\right]}^{2}dx.
\end{align}
}%
Applying~\eqref{eq:cs} to $T_{2}$ in~\eqref{eq:T12} we have:
\par\nobreak
{\small
\begin{align}
 &{\left( \Ep_{c,c'} \left[\frac{1}{Z_{c}Z_{c'}} \expip{h}{1}{c} \expip{t}{2}{c'} \vec{1}_{\bar{F}} \right] \right)}^{2} \nonumber \\
 \leq & \left( \Ep_{c,c'}\left[ \frac{1}{Z^{2}_{c}} \expip{h}{1}{c}^{2}  \vec{1}_{\bar{F}} \right] \right) \times \nonumber \\
  &\left( \Ep_{c,c'}\left[ \frac{1}{Z^{2}_{c'}} \expip{t}{2}{c'}^{2}  \vec{1}_{\bar{F}} \right] \right)  \nonumber \\
  = &\left( \Ep_{c}\left[ \frac{1}{Z^{2}_{c}} \expip{h}{1}{c}^{2}  \Ep_{c'|c}\left[\vec{1}_{\bar{F}}\right] \right] \right) \times \nonumber \\
 &  \left( \Ep_{c'}\left[ \frac{1}{Z^{2}_{c'}} \expip{t}{2}{c'}^{2}  \Ep_{c|c'}\left[\vec{1}_{\bar{F}}\right] \right] \right)  \label{eq:twoterms}
\end{align}
}
Note that $Z_{c} \geq 1$ because $Z_{c}$ is the sum of positive numbers and if $\vec{h}\T\mat{R}_{1}\vec{c} > 0$ for at least one of the $h \in \cV$,
then the total sum will be greater than 1. Therefore, by dropping $Z_{c}$ term from the denominator we can further increase the first term in~\eqref{eq:twoterms} as given by~\eqref{eq:maxz}.
{\small
\begin{align}
 \label{eq:maxz}
 & \Ep_{c} \left[ \frac{1}{Z^{2}_{c}} \expip{h}{1}{c}^{2} \Ep_{c'|c}\left[\vec{1}_{\bar{F}}\right] \right] \nonumber \\
 &\leq \Ep_{c}\left[\expip{h}{1}{c}^{2} \Ep_{c'|c}\left[\vec{1}_{\bar{F}}\right]\right]
\end{align}
}%

Let us split the expectation on the R.H.S.\ of~\eqref{eq:maxz} into two cases depending on whether $\vec{h}\T\mat{R}_{1}\vec{c} > 0$ or otherwise, indicated respectively by $\vec{1}_{(\vec{h}\T\mat{R}_{1}\vec{c} > 0)}$ and $\vec{1}_{(\vec{h}\T\mat{R}_{1}\vec{c} \leq 0)}$.
\begin{align}
 &\Ep_{c}\left[\expip{h}{1}{c}^{2} \Ep_{c'|c}\left[\vec{1}_{\bar{F}}\right]\right] \nonumber \\
 &= \Ep_{c}\left[\expip{h}{1}{c}^{2} \vec{1}_{(\vec{h}\T\mat{R}_{1}\vec{c} > 0)} \Ep_{c'|c} \left[\vec{1}_{\bar{F}}\right] \right] \nonumber \\ 
 &+ \Ep_{c}\left[\expip{h}{1}{c}^{2} \vec{1}_{(\vec{h}\T\mat{R}_{1}\vec{c} \leq 0)} \Ep_{c'|c} \left[\vec{1}_{\bar{F}}\right] \right] \label{eq:twoexps}
\end{align}
The second term of~\eqref{eq:twoexps} is upper bounded by
\begin{align}
\Ep_{c,c'}\left[\vec{1}_{\bar{F}}\right] \leq \exp\left(-\Omega(\log^{2}n)\right)
\end{align}
The first term of~\eqref{eq:twoexps} can be bounded as follows:
\begin{align}
&\Ep_{c}\left[\expip{h}{1}{c}^{2} \vec{1}_{(\vec{h}\T\mat{R}_{1}\vec{c} > 0)} \Ep_{c'|c} \left[\vec{1}_{\bar{F}}\right] \right] \nonumber \\
&\leq \Ep_{c} \left[ {\exp(\alpha \vec{h}\T \mat{R}_{1} \vec{c})}^{2}  \vec{1}_{(\vec{h}\T\mat{R}_{1}\vec{c} > 0)} \Ep_{c'|c} \left[\vec{1}_{\bar{F}}\right] \right] \nonumber \\
&\leq \Ep_{c}  \left[ {\exp(\alpha \vec{h}\T \mat{R}_{1} \vec{c})}^{2} \Ep_{c'|c} \left[\vec{1}_{\bar{F}}\right] \right]
\end{align}
where $\alpha > 1$.
Therefore, it is sufficient to bound $\Ep_{c}  \left[ {\exp(\alpha \vec{h}\T \mat{R}_{1} \vec{c})}^{2} \Ep_{c'|c} \left[\vec{1}_{\bar{F}}\right] \right] $ when $\norm{\vec{h}} = \Omega(\sqrt{d})$.

Let us denote by $z$ the random variable $2\vec{h}\T\mat{R}_{1}\vec{c}$. Moreover, let $r(z) = \Ep_{c'|z} [\vec{1}_{\bar{F}}]$, which is a function of $z$ between $[0,1]$.
We wish to upper bound $\Ep_{c}[\exp(z)r(z)]$. The worst-case $r(z)$ can be quantified using a continuous version of Abel's inequality (proved as Lemma A.4 in~\cite{arora2015rand}), we can upper bound $\Ep_{c}\left[\exp(z)r(z)\right]$ as follows:
\begin{align}
 \label{eq:au1}
 \Ep_{c}\left[\exp(z)r(z)\right] \leq \Ep\left[\exp(z)\vec{1}_{[t,+\infty]}(z)\right]
\end{align}
where $t$ satisfies that $\Ep_{c}[\vec{1}_{[t,+\infty]}(z)] = \Pr[z \geq t] = \Ep_{c}[r(z)] \leq \exp(-\Omega(\log^{2}n))$.
Here, $\vec{1}_{[t,+\infty]}(z)$ is a function that takes the value $1$ when $z \geq t$ and zero elsewhere.
Then, we claim $\Pr_{c}[z \geq t] \leq \exp(-\Omega(\log^{2}n))$ implies that $t \geq \Omega(\log^{.9} n)$.

If $c$ was distributed as $\cN(0, \frac{1}{d}\mat{I})$, this would be a simple tail bound. However, as $c$ is distributed uniformly on the sphere, this requires special care, and the claim follows by applying the tail bound for the spherical distribution given by Lemma A.1 in \cite{arora2015rand} instead.
Finally, applying Corollary A.3 in \cite{arora2015rand}, we have:
\begin{align}
 \Ep[\exp(z)r(z)] \leq & \ \Ep[\exp(z)\vec{1}_{[t,+\infty]}(z)] \nonumber \\
 = & \exp(-\Omega(\log^{1.8}n))
\end{align}
From a similar argument as above we can obtain the same bound for $c'$ as well. Therefore, $T_{2}$ in~\eqref{eq:T12} can be upper bounded as follows:
\par\nobreak
{\small
\begin{align}
 & \Ep_{c,c'} \left[\frac{1}{Z_{c}Z_{c'}} \expip{h}{1}{c} \expip{t}{2}{c'} \vec{1}_{\bar{F}} \right]  \nonumber \\
 & \leq  {\left( \Ep_{c}\left[ \frac{1}{Z^{2}_{c}} \expip{h}{1}{c}^{2}  \Ep_{c'|c}\left[\vec{1}_{\bar{F}}\right] \right] \right)}^{1/2} \times \nonumber \\ 
 & {\left( \Ep_{c'}\left[ \frac{1}{Z^{2}_{c'}} \expip{t}{2}{c'}^{2}  \Ep_{c|c'}\left[\vec{1}_{\bar{F}}\right] \right] \right)}^{1/2} \nonumber \\
  &\leq \exp(-\Omega(\log^{1.8}n))
\end{align}
}%
Because $n = |\cV|$, the size of the entity vocabulary, is large (ca. $n > 10^{5}$) in most knowledge graphs, we can ignore the $T_{2}$ term in~\eqref{eq:T12}.

Combining the above analysis of $T_2$ term with~\eqref{eq:T12} we obtain an upper bound for $p(h,t \mid r)$ given by~\eqref{eq:prob-upper}.
\par\nobreak
{\small
\begin{align}
 \label{eq:prob-upper}
  &p(h,t \mid R)  \leq {(1+\epsilon_{z})}^{2} \frac{1}{Z^{2}} \Ep_{c,c'} \left[ \expip{h}{1}{c} \expip{t}{2}{c'} \vec{1}_{F} \right] \nonumber \\
  & \ \ \  \ \ \ \  \ \ \ \ \ \ \ \ \ \ \ \  \ +|\cD|\exp(-\Omega(\log^{1.8}n)) \nonumber \\
 &=  {(1+\epsilon_{z})}^{2} \frac{1}{Z^{2}} \Ep_{c,c'} \left[ \expip{h}{1}{c} \expip{t}{2}{c'} \right] + \delta_{0}
\end{align}
}%
where $|\cD|$ is the number of relational tuples $(h,r,t)$ in the KB and
$\delta_{0} = |\cD|\exp(-\Omega(\log^{1.8}n)) \leq \exp(-\Omega(\log^{1.8}n))$ by the fact that $Z \leq \exp(2\kappa)n = O(n)$, where $\kappa$ is the upper bound on $\vec{h}\T\mat{R}_{1}\vec{c}$ and $\vec{t}\T\mat{R}_{2}\vec{c'}$, which is regarded as a constant.

On the other hand, we can lower bound $p(h,t \mid r)$ as given by~\eqref{eq:prob-lower}.
\par\nobreak
{\small
\begin{align}
 \label{eq:prob-lower}
 & p(h,t \mid R) \geq {(1-\epsilon_{z})}^{2} \frac{1}{Z^{2}} \Ep_{c,c'} \left[ \expip{h}{1}{c} \expip{t}{2}{c'} \vec{1}_{F} \right] \nonumber \\
 &\geq {(1-\epsilon_{z})}^{2} \frac{1}{Z^{2}} \Ep_{c,c'} \left[ \expip{h}{1}{c} \expip{t}{2}{c'}\right] \nonumber \\ &-|\cD|\exp(-\Omega(\log^{1.8}n)) \nonumber \\
 & \geq {(1-\epsilon_{z})}^{2} \frac{1}{Z^{2}} \Ep_{c,c'} \left[ \expip{h}{1}{c} \expip{t}{2}{c'}\right] - \delta_{0}
\end{align}
}%
Taking the logarithm of both sides, from~\eqref{eq:prob-upper} and~\eqref{eq:prob-lower}, the multiplicative error translates to an additive error given by~\eqref{eq:add-error}.
\par\nobreak
{\small
\begin{align}
 \label{eq:add-error}
 &\log p(h,t \mid R) = \log \left( \Ep_{c,c'} \left[ \expip{h}{1}{c}\expip{t}{2}{c'}\right] \pm \delta_{0} \right)  \nonumber \\
 &\ \ \ \ \  \ \  \ \ \  \ \ \ \  \ \ \ \ \ \ \ \  \ \ -2\log Z + 2\log(1\pm \epsilon_{z}) \nonumber \\
 &= \log \left( \Ep_{c} \left[ \expip{h}{1}{c} \Ep_{c'|c} \left[ \expip{t}{2}{c'} \right] \right] \pm \delta_{0} \right) \nonumber \\
 & \ \ \  \ \ \ \ \ \  - 2\log Z + 2\log(1\pm \epsilon_{z}) \nonumber \\
 &= \log \left( \Ep_{c} \left[ \expip{h}{1}{c} A(c)  \right]\pm \delta_{0}\right)  \nonumber \\
 &\ \ \  \ \ \ \ \ \ -2\log Z + 2\log(1\pm \epsilon_{z})
\end{align}
}%
where $A(c) \coloneqq \Ep_{c'|c}\left[ \expip{t}{2}{c'}\right]$.

We assumed that $\vec{c}$ and $\vec{c}'$ are on the unit sphere and $\mat{R}_{1}$ and $\mat{R}_{2}$ to be orthogonal matrices.
Therefore, $\mat{R}_{1}\vec{c}$ and $\mat{R}_{2}\vec{c}'$ are also on the unit sphere.
Moreover, if we let the upper bound of the $\ell_{2}$ norm of the entity embeddings to be $\kappa' \sqrt{d}$, then we have $\norm{\vec{h}} \leq \kappa' \sqrt{d}$
and $\norm{\vec{t}} \leq \kappa' \sqrt{d}$.
Therefore, we have
\begin{align}
\langle\mat{R}_{1}\vec{h}, \vec{c}' - \vec{c}\rangle \leq \norm{\vec{h}}\norm{\vec{c} - \vec{c}'} \leq \kappa' \sqrt{d}\norm{\vec{c} - \vec{c}'}
\end{align}
Then, we can upper bound $A(c)$ as follows:
\par\nobreak
{\small
\begin{align}
A(c)  &= \Ep_{c'|c}\left[ \expip{t}{2}{c'}\right] \nonumber \\
&= \expip{t}{2}{c} \Ep_{c'|c}\left[\exp\left(\vec{t}\T\mat{R}_{2}(\vec{c}'-\vec{c})\right)\right] \nonumber \\
 &\leq \expip{t}{2}{c} \Ep_{c'|c}\left[\exp\left(\kappa' \sqrt{d}\norm{\vec{c}'-\vec{c}}\right)\right]  \nonumber \\
 &\leq (1+\epsilon_{2}) \expip{t}{2}{c}
\end{align}
}%
For some $\epsilon_{2} > 0$.
The last inequality holds because
\par\nobreak
{\small
\begin{align}
 &\Ep_{c|c'}\left[\exp\left(\kappa' \sqrt{d}\norm{\vec{c}'-\vec{c}}\right)\right] \nonumber \\ 
 &=  \int \exp\left(\kappa' \sqrt{d}\norm{\vec{c}'-\vec{c}}\right) p(c'|c) dc' \nonumber \\
 &= \underbrace{\exp(\kappa'\sqrt{d})}_{\geq 1} \underbrace{\int \exp(\norm{\vec{c}-\vec{c}'}) p(c'|c)dc'}_{\geq 1} 
 \nonumber \\ &= 1+\epsilon_{2} \label{eq:A-upper}
\end{align}
}%
To obtain a lower bound on $A(c)$ from the first-order Taylor approximation of $\exp(x) \geq 1 + x$ we observe that:
\begin{align}
  &\Ep_{c|c'}\left[\exp\left(\kappa' \sqrt{d}\norm{\vec{c}'-\vec{c}}\right)\right] \nonumber \\
  &+  \Ep_{c|c'}\left[\exp\left(-\kappa' \sqrt{d}\norm{\vec{c}'-\vec{c}}\right)\right] \geq 2 .
\end{align}
Therefore, from our model assumptions we have
\par\nobreak
{\small
\begin{align}
 \Ep_{c|c'}\left[\exp\left(-\kappa' \sqrt{d}\norm{\vec{c}'-\vec{c}}\right)\right] \geq 1 - \epsilon_{2}
\end{align}
}%
Hence,
\par\nobreak
{\small
\begin{align}
A(c)  &= \expip{t}{2}{c} \Ep_{c'|c}\left[\exp\left(\vec{t}\T\mat{R}_{2}(\vec{c}'-\vec{c})\right)\right] \nonumber \\
& \geq  \expip{t}{2}{c} \Ep_{c'|c}\left[\exp\left(-\kappa' \sqrt{d}\norm{\vec{c}'-\vec{c}}\right)\right]  \nonumber \\
&\geq (1-\epsilon_{2}) \expip{t}{2}{c} \label{eq:A-lower}
\end{align}
}%
Therefore, from~\eqref{eq:A-upper} and~\eqref{eq:A-lower} we have
\begin{align}
 \label{eq:A}
 A(c) = (1 \pm \epsilon_{2}) \expip{t}{2}{c}
\end{align}
Plugging $A(c)$ back in~\eqref{eq:add-error} results in $\log p(h,t \mid r)$ equal to: 
\par\nobreak
{\small
\begin{align}
\label{eq:A-last}
  \log& \left( \Ep_{c} \left[ \expip{h}{1}{c} A(c) \right]\pm \delta_{0}\right) -  2\log Z + 2\log(1\pm \epsilon_{z}) \nonumber \\
  =& \log \left( \Ep_{c} \left[ \expip{h}{1}{c} (1 \pm \epsilon_{2}) \expip{t}{2}{c}  \right]\pm \delta_{0}\right)  \nonumber \\
  &-2\log Z + 2\log(1\pm \epsilon_{z}) \nonumber\\ 
 =& \log \left( \Ep_{c} \left[ \expip{h}{1}{c}  \expip{t}{2}{c}  \right]\pm \delta_{0}\right)   \nonumber \\
 &-2\log Z + 2\log(1\pm \epsilon_{z}) + \log(1 \pm \epsilon_{2}) \nonumber \\ 
  =& \log \left( \Ep_{c} \left[ \exp \left(\vec{h}\T\mat{R}_{1}\vec{c} + \vec{t}\T\mat{R}_{2}\vec{c} \right)  \right]\pm \delta_{0}\right)  \nonumber \\
  &-2\log Z + 2\log(1\pm \epsilon_{z}) + \log(1 \pm \epsilon_{2}) \nonumber \\ 
   =& \log \left( \Ep_{c} \left[ \exp \left(\left(\mat{R}_{1}\T\vec{h} + \mat{R}_{2}\T\vec{t} \right)\T\vec{c}\right)  \right]\pm \delta_{0}\right)  \nonumber \\
   &-2\log Z + 2\log(1\pm \epsilon_{z}) + \log(1 \pm \epsilon_{2}) 
\end{align}
}%

Note that $\vec{c}$ has a uniform distribution over the unit sphere. In this case, from Lemma A.5 in~~\cite{arora2015rand},~\eqref{eq:spherical-exp} holds approximately.
\par\nobreak
{\small
\begin{align}
\label{eq:spherical-exp}
 &\Ep_{c}\left[\exp \left(\left(\mat{R}_{1}\T\vec{h} + \mat{R}_{2}\T\vec{t} \right)\T\vec{c}\right) \right] \nonumber \\
 &= (1 \pm \epsilon_{3}) \exp \left( \frac{ \norm{\mat{R}_{1}\T\vec{h} + \mat{R}_{2}\T\vec{t}}^{2}_{2}}{2d} \right)
\end{align}
}%
where $\epsilon_{3} = \tilde{O}(1/d)$.
Plugging~\eqref{eq:spherical-exp} in~\eqref{eq:A-last} we have that
\par\nobreak
{\small
\begin{align}
  \log p(h,t \mid R) &= \frac{ \norm{\mat{R}_{1}\T\vec{h} + \mat{R}_{2}\T\vec{t}}_{2}^{2}}{2d} + \nonumber \\
  &O(\epsilon_{z}) + O(\epsilon_{2}) + O(\epsilon_{3}) + O(\delta'_{0}) - 2\log Z
\end{align}
}%
where $\delta'_{0} = \delta_{0}$.

${\left( \Ep_{c} \left[\exp \left( (\mat{R}_{1}\T\vec{h} + \mat{R}_{2}\T\vec{t})\T\vec{c} \right) \right] \right)}^{-1} = \exp( -\Omega(\log^{1.8}n))$. 
Therefore, $\delta_{0}'$ can be ignored.
Note that $\epsilon_{3} = \tilde{O}(1/d)$ and $\epsilon_{z} = \tilde{O}(1/\sqrt{n})$ by assumption.
Therefore, we obtain that
\par\nobreak
{\small
\begin{align}
 \log p(h,t \mid R) &= \frac{ \norm{\mat{R}_{1}\T\vec{h} + \mat{R}_{2}\T\vec{t}}_{2}^{2}}{2d} + \nonumber \\
 & O(\epsilon_{z}) + O(\epsilon_{2}) + \tilde{O}(1/d) - 2\log Z
\end{align}
}%
\section{Learning with Multiple Negative Triples}
\label{sec:MultipleNeg}
This approach can be easily extended to learn from multiple negative triples as follows.
%We want to show how the marginal loss learning objective derived in Section~\ref{sec:learn} can be extended to learn from more than one negative triple per each positive triple. This formulation leads to \emph{rank-based} loss objective used in prior work on KGE. Considering that negative triples are generated via random perturbation, it is important to consider multiple negative triples during training to better estimate the classification boundary.
Let us consider that we are given a positive triple, $(h,R,t)$ and a set of $K$ negative triples $\{(h'_{k}, R, t'_{k})\}_{k=1}^{K}$.
We would like our model to assign a probability, $p(h,t \mid R)$, to the positive triple that is higher than that assigned to any of the negative triples.
This requirement can be written as \eqref{eq:max-prob}.
\begin{align} \label{eq:max-prob}
 p(h,t|R) \geq \Max{k=1, \ldots, K} p(h'_{k}, t'_{k} \mid R)
\end{align}
We could further require the ratio between the probability of the positive triple and maximum probability over all negative triples to be greater than a threshold $\eta \geq 1$ to make the requirement of \eqref{eq:max-prob} to be tighter.
\begin{align} 
\label{eq:max-prob2}
 \frac{p(h, t \mid R)}{ \Max{k=1, \ldots, K} p(h'_{k}, t'_{k} \mid R)} \geq \eta
\end{align}
By taking the logarithm of \eqref{eq:max-prob2} we obtain
{\small
\begin{align} \label{eq:max-prob3}
 \log p\left(h, t \mid R\right) - \log \left( \Max{k=1, \ldots, K} p\left(h'_{k}, t'_{k} \mid R\right) \right) \geq \log(\eta)
\end{align}
}%
Therefore, we can define the marginal loss for a misclassification as follows:
{\small
\begin{align}
 \label{eq:max-loss}
& L\left( \left(h,R,t\right), \{\left(h'_{k}, R, t'_{k}\right)\}_{k=1}^{K} \right) = \nonumber \\
& \max \Big(0, \log \left( \Max{k=1, \ldots, K} p(h'_{k}, t'_{k} \mid R) \right) + \nonumber \\
 &+ \log\left(\eta\right) -  \log p\left(h, t \mid R\right)\Big)
\end{align}
}%
However, from the monotonicity of the logarithm we have $\forall x_{1}, x_{2} > 0$, if $\log(x_{1}) \geq \log(x_{2})$ then $x_{1} \geq x_{2}$.
Therefore, the logarithm of the maximum can be replaced by the maximum of the logarithms in \eqref{eq:max-loss} as shown in \eqref{eq:max-loss2}.
\par\nobreak
{\small
\begin{align}
 \label{eq:max-loss2}
 &L\left( \left(h,R,t\right), \{\left(h'_{k}, R, t'_{k}\right)\}_{k=1}^{K} \right) = \nonumber \\
 &\max \Big(0,  \Max{k=1, \ldots, K} \log \left( p\left(h'_{k}, t'_{k} \mid R\right) \right) \nonumber \\
 &+ \log\left(\eta\right) -  \log p\left(h, t \mid R\right)\Big)
\end{align}
}
By substituting \eqref{eq:joint-prob} for the probabilities in \eqref{eq:max-loss2} we obtain the rank-based loss given by \eqref{eq:rank-loss}.
\par\nobreak
{\small
\begin{align}
 \label{eq:rank-loss}
 & L\left( (h,R,t), \{(h'_{k}, R, t'_{k})\}_{k=1}^{K} \right) = \nonumber \\
 &\max \Big(0,  2d\log(\eta) + \Max{k=1, \ldots, K}  \norm{\mat{R}_{1}\T \vec{h}'_{k} + \mat{R}_{2}\T\vec{t}'_{k}}_{2}^{2}  \nonumber \\
 &-  \norm{\mat{R}_{1}\T\vec{h} + \mat{R}_{2}\T\vec{t}}_{2}^{2} \Big)
\end{align}
}
In practice, we can use $p(h'_{k},t'_{k} \mid R)$ to select the negative triple with the highest probability for training with the positive triple.

\section{Training Details}
\label{sec:TrainDetails}
\begin{table}[t]
\small
\centering
\scalebox{0.9}{
\begin{tabular}{l l l l l l}\toprule
Dataset & \#R & \#E & Train & Test & Val. \\ \midrule
% FB15K	& 1,345	& 14,951 &	483,142	& 59,071 &	50,000 \\
FB15K237	& 237	& 14,541 &	272,115	& 17,535 &	20,466 \\
% WN18	& 18	& 40,943 &	141,442	& 5,000	& 5,000 \\
WN18RR	& 11	& 40,943 &86,835& 3,134&3,034\\
%WN11	& 11	& 38,588 &	112,581	& 10,544 &	2,609 \\
FB13	& 13	& 75,043 &	316,232	& 23,733 &	5,908 \\
\bottomrule
\end{tabular}}
\caption{Statistics of the datasets}
\label{tbl:datasets}
\end{table}

The statistics of the benchmark datasets are shown in~\autoref{tbl:datasets}.
We selected the initial learning rate ($\alpha$) for SGD in $\{0.01,0.001\}$, the regularisation coefficients ($\lambda_{1}, \lambda_{2}$) for the orthogonality constraints of relation matrices in $\{0,1,10,100\}$. 
The number of randomly generated negative triples $n_{\textrm{neg}}$ for each positive example is varied in $\{1,10,20,50,100\}$
and $d \in \{50, 100\}$.
%Optimal hyperparameter settings were: $\lambda_{1} = \lambda_{2} =10$, $n_{\textrm{neg}} =100$ for all the datasets, $\alpha=0.001$ for FB15K, FB15K237 and FB13, $\alpha=0.01$ for WN18, WN18RR and WN11. 
Optimal hyperparameter settings were: $\lambda_{1} = \lambda_{2} =10$, $n_{\textrm{neg}} =100$ for all the datasets, $\alpha=0.001$ for FB15K237 and FB13, $\alpha=0.01$ for WN18RR. 
%For FB15K237 and WN18RR $d=100$ was the best, whereas for all other datasets $d=50$ performed best.
For FB15K237 and WN18RR $d=100$ was the best, whereas for FB13 $d=50$ performed best.
Negative triples are generated by replacing a head or a tail entity in a positive triple by a randomly selected entity and learn KGEs.
We train the model until convergence or at most 1000 epochs over the training data where each epoch is divided into 100 mini-batches. The best model is selected by early stopping based on the performance of the learnt embeddings on the validation set (evaluated after each 20 epochs). 
%
%Setting the $\ell_{2}$ regularisation coefficients $\lambda_{1} = \lambda_{2} = 1$ and using $10$ negative instances in each minibatch resulted in the optimal performance for \textbf{RelWalk}, measured on the validation datasets.

\end{document}